\documentclass[10pt, a4paper, logo, twocolumn, nonumbering]{deepmind}

\usepackage[utf8]{inputenc}
\usepackage[T1]{fontenc}
\usepackage[draft]{hyperref}
\usepackage{amsfonts}
\usepackage{amssymb}
\usepackage{amsmath}
\usepackage{amsthm}
\usepackage{nicefrac}
\usepackage{microtype}
\usepackage{graphicx}
\usepackage{algorithm, algorithmic}
\usepackage{mathpazo}
\usepackage{xcolor}
\usepackage[authoryear, round]{natbib} 

\def\w{\mathbf{w}}

\def\e{\mathrm{e}}

\def\E#1{\mathbb{E}\left[#1\right]}

\newtheorem{theorem}{Theorem}

\title{General non-linear Bellman equations}

\author{
Hado van Hasselt,
John Quan,
Matteo Hessel,
Zhongwen Xu,
Diana Borsa,
Andre Barreto,
\\DeepMind, London, UK}

\begin{abstract}
We consider a general class of non-linear Bellman equations. These open up a design space of algorithms that have interesting properties, which has two potential advantages. First, we can perhaps better model natural phenomena. For instance, hyperbolic discounting has been proposed as a mathematical model that matches human and animal data well, and can therefore be used to explain preference orderings. We present a different mathematical model that matches the same data, but that makes very different predictions under other circumstances. Second, the larger design space can perhaps lead to algorithms that perform better, similar to how discount factors are often used in practice even when the true objective is undiscounted. We show that many of the resulting Bellman operators still converge to a fixed point, and therefore that the resulting algorithms are reasonable and inherit many beneficial properties of their linear counterparts.
\end{abstract}

\begin{document}
\maketitle

Reinforcement learning \citep{SuttonBarto:2018} is a framework for decision making under uncertainty, where a learning system, referred to as the \textit{agent}, must learn to act in its surrounding \textit{environment} so as to maximize such sum of discounted rewards $G_t = \sum_{n=1}^\infty \gamma^n R_{t+n}$, where $\gamma\in [0,1]$ is a discount factor that discounts the importance of future rewards over immediate rewards. According to the reward hypothesis \citep{SuttonBarto:2018}, any goal can be formulated as such a sum. The behaviour of an agent is characterized by a policy $\pi$: a (possibly stochastic) mapping from environment \textit{states} to a set of available \textit{actions}. 

Richard Bellman noted that the expected values $v(s) = \E{ G_t \mid S_t = s}$ of cumulative rewards admit a recursive formulation \citep{Bellman:57}.  Specifically, the value $v_{\pi}(s) = \E{ G_t \mid A_n \sim \pi(S_n), \forall n \ge t}$ of a state $s$ under a policy $\pi$ can be written as
\[
v_\pi(s)
~~=~~
\E{ R_{t+1} + \gamma v_\pi(S_{t+1}) \mid S_t = s, A_t \sim \pi(S_t) } \,,
\]
where $R_{t+1}$ is the stochastic reward received on the transition from state $S_t$ to state $S_{t+1}$ after executing action $A_t$, as sampled according to policy $\pi$. Expected returns conditioned on a specific first action, or action-values, can also be expressed recursively $q_\pi(s,a) = \E{ R_{t+1} + \gamma q_\pi(S_{t+1}, \pi(S_{t+1})) \mid S_t = s, A_t = a }$. This formulation is useful because an improved policy $\pi'$ can be directly inferred from the action values for a policy $\pi$. For instance, the agent may act according to the greedy policy of selecting, in each state, the action with the highest action value under $\pi$.

Given a complete specification of the environment dynamics, linear Bellman equations such as described above may be written as a system of linear equations, which admit closed-form solutions, although the computational cost of such solutions scales with the cube of the number of states. However, efficient dynamic programming algorithms \citep{Bellman:57, Howard:1960, Puterman:1994} are also available to solve these Bellman equations iteratively. These equation can also be solved through sample-based temporal-difference learning \citep{Sutton:1988, Watkins:1989}, which uses the same recursion to create efficient algorithms that can learn model-free (i.e. without relying with a specification of the environment's dynamics), directly from data obtained through interaction with the environment.

\section{Non-linear Bellman equations}
The recursive formulation is appealing as it allows for algorithms that can sample and update estimates of these values independently of temporal span \citep{vanHasseltSutton:2015}.  The canonical formulation limits the modelling power of Bellman equations to cumulative rewards that are discounted exponentially: the weight on a reward $t$ steps in the future will be discounted with a factor $\gamma^t$.

We consider a broader class of Bellman equations that are non-linear in the rewards and future values:
\[
v(s) = \E{ f(R_{t+1}, v(S_{t+1})) \mid  S_t = s, A_t \sim \pi(S_t) } \,.
\]
This generalises the standard Bellman equation which is obtained for $f(r, v) = r + \gamma v$. We conjecture that the additional flexibility this provides can be useful for at least two purposes: 1) to model a wider range of natural phenomena, including to explain human and animal behaviour, and 2) to allow more efficient learning algorithms for prediction and control by widening the flexibility in design choices for such algorithms.

Three slightly more specific formulations will be of particular interest to us: non-linear transformations to the reward ($f(r, v) = g(r) + \gamma v$), to the value ($f(r, v) = r + g(v)$), and to the target as a whole ($f(r, v) = h(r + g(v))$), where, in each case, $g: \mathbb{R} \to \mathbb{R}$ and $h: \mathbb{R} \to \mathbb{R}$ are scalar functions.  In the latter case, we may choose $g(v) = \gamma h^{-1}(v)$, such that $f(r, v) = h(r + \gamma h^{-1}(v))$, where $h$ is a squashing function as will be defined later, such that $v$ can be roughly interpreted as estimating a squashed estimate of the expected return.

In all these cases the Bellman equations \textit{define} the values of the states. The value should therefore be considered a function of the chosen non-linearity. Note that this is also true for the standard discounted formulation: the value of a state under a given discount $\gamma$ differs from the value under a discount $\gamma'\ne\gamma$, and neither is necessarily equivalent to the undiscounted objective if $\gamma'<1$ and $\gamma<1$, neither in terms of value nor in terms of induced behaviour.

\subsection{Pre-existing non-linear Bellman equations}

Humans and animals seem to exhibit a different type of weighting of the future than would emerge from the standard linear Bellman equation which leads to exponential discounting when unrolled multiple steps because of the repeated multiplication with $\gamma$. One consequence is that the preference ordering of two different rewards occurring at different times can reverse, depending on how far in the future the first reward is.  For instance, humans may prefer a single sparse reward of $+1$ (e.g., \$1) now over a reward of $+2$ (e.g., \$2) received a week later, but may also prefer a reward of $+2$ received after 20 weeks over a reward of $+1$ after 19 weeks.  Such preferences reversals \citep{Thaler:1981, Ainslie:1981, Green:1994} have been observed in human and animal studies, but cannot be predicted from exponential discounting.  Instead, hyperbolic discounting has been proposed as a well-fitting mathematical model, where a reward in $t$ steps is discounted as $R_t/(1 + k t)$, or some variation of this equation.

It has been shown that at least some of the data can be explained with a recursive formulation, called HDTD \citep{Alexander:2010}, that uses a recursion $v(s) = \E{(R_{t+1} + v(S_{t+1}))/(1 + k v(S_{t+1})}$. Note that this is a non-linear Bellman equation, due to the division by the value of $S_{t+1}$. Interestingly, mixing values for multiple exponential discounts \citep[as discussed by][]{Sutton:1995} can also closely approximate hyperbolic discounting \citep{Kurth:2009,Fedus:2019}

Separately, discounting is a useful tool to increase control performance.  In modern reinforcement learning applications, the discount factor is rarely set to $\gamma=1$, even if the goal is to optimise the average total return per episode \citep[e.g.,][]{Mnih:2015,vanHasselt:2016DDQN,Wang:2016,Hessel:2018}.  Indeed, also when learning this factor from data, the learnt value often stays below $\gamma=1$ \citep{Xu:2018}.  This makes intuitive sense: it can be substantially easier to learn a policy of control that is somewhat myopic: because we get to make more decisions later, when the sliding horizon into the future will have moved forward with time, the resulting behaviour is not necessarily much worse than the optimal policy for the far-sighted undiscounted formulation. In other words, it can be useful to learn a proxy for the true objective if the proxy is easier to learn.

In some cases it was found to be beneficial for performance to add a non-linear transform to the updates \citep{Pohlen:2018,Kapturowski:2019}: $v(s) = \E{ h( R_{t+1} + \gamma h^{-1}(v(S_{t+1}))) }$. This effectively scales down the values before updating the parametric function (a multi-layer neural network) toward the scaled-down value, and then scale the values back up with $h^{-1}(v(S_{t+1}))$ before using these in the temporal difference update.  The intuition is that it can be easier for the network to accurately represent the values in this transformed space, especially if the true values can have quite varying scales. For instance, in the popular Atari 2600 benchmark \citep{Bellemare:2013}, the different games can have highly varying reward scales, which can be hard to deal with for learning algorithms that do not do anything special to deal with this \citep{vanHasselt:2016PopArt}.

Finally, we note that many distributional reinforcement learning algorithms \citep{Bellemare:2017,Dabney:2018} can be interpreted as optimising non-linear Bellman equations.  It may be helpful to view these as part of a larger set of possible transformations of the underlying reward signal \citep[cf.][]{Rowland:2019}.

\section{Analysis of non-linear TD algorithms}
\def\w{\mathbf{w}}
We now examine properties of temporal-difference algorithms \citep{Sutton:1988} derived from non-linear Bellman equations.  When $v_{\w}$ is a parametric function with parameter $\w$, these algorithms have the following form
\[
v_{\w}(S_t) = v_\w(S_t) + \alpha (Y_{t+1} - v_\w(S_t)) \nabla_{\w} v_{\w}(S_t) \,,
\]
where
\[
Y_{t+1} \equiv f(R_{t+1}, v_\w(S_{t+1})) 
\]
is some function of the reward and next state value.

This formulation subsumes the tabular case, where $v_{\w}(s) = w_s$ is just the scalar value $w_s$ at the cell in the table corresponding to state $s$. The first question we can ask is whether the operator $T^f$ defined by $(T^f v)(s) = \E{ f(R_{t+1}, v(S_{t+1}) \mid  S_t = s, A_t \sim \pi(S_t) }$ is a contraction, which is a sufficient condition for convergence under standard assumptions \citep{Bertsekas:96}. This will, of course, depend on the choice of $f$.

\subsection{Reward transforms}
\def\argmin{\text{arg min}}
For $f(r, v) = g(r) + \gamma v$ (\emph{reward transforms}), all the standard convergence results will hold (as long as $g$ is bounded). For instance, the tabular algorithms will converge to $v(s) = \E{ g(R_{t+1}) + \gamma v(S_{t+1}) }$, and the linear algorithms will converge to the minimum of the mean-squared projected Bellman error \citep{Sutton:2009icml, SuttonBarto:2018} $\| \Pi T^f v_\w - v_\w \|_{d^\pi}$, where $\Pi$ is a projection onto the space of representable value functions, such that $\Pi v = v_{\w^*}$ where $\w^* = \argmin_\w \|v_\w - v\|_{d^\pi}$, and $\| \cdot \|_{d^\pi}$ is a weighted norm $\sum_s d^\pi(s)(\cdot)^2$, where $d^\pi(s) = \lim_{t\to\infty} P(S_t = s)$ is the steady-state probability of being in state $s$, which is assumed to be well-defined (e.g., the MDP is ergodic).

Reward transforms are more interesting than they perhaps at first appear.  For instance, a reward transform exists that approximates a hyperbolic discount factor quite well, in the context of sparse rewards, but with a standard exponential discounted value.
\def\e{\mathrm{e}}

Concretely, let
\[
H_t = \sum_{n=0} R_{t+n+1}/(1 + k n)
\]
denote a hyperbolically-discounted return with discount parameter $k$. Consider episodes with a single non-zero reward on termination. This matches the setting of ``do you prefer \$X now or \$Y later'', where we do not allow the possibility of both rewards happening. If the terminal reward occurred at time step $t+1$, this implies $H_0 = R_{t+1}/(1 + k t)$. Then, for any $t$, we would prefer a later non-zero reward $R_{t+1}$ to an immediate reward of $r$ if and only if $R_{t+1}/r > 1 + k T$.  This induces a specific preference ordering on sparse rewards.
\begin{theorem}
Consider an exponentially-discounted return
\[
G_t = \sum_{n=0}^\infty \gamma^n g(R_{t+n+1}) = g(R_{t+1}) + \gamma G_{t+1} \,,
\]
with non-linear reward transform $g$. In the episodic setting, where only the terminal reward is non-zero, then for any exponential discount with parameter $\gamma \in (0, 1)$ and hyperbolic discount with parameter $k > 0$, a reward transform $g$ exists that induces the exact same reward ordering for the geometric return $G_0$ and the hyperbolic return $H_0$, in the sense that for any episode terminating (randomly) after $T+1$ steps with reward $R_{T+1}$ we prefer a return $G_0$ with sparse non-zero reward $R_{T+1} > 0$ to an immediate reward $r$ if and only if $R_{T+1}/r > 1 + k T$.
Concretely, this holds for the transform defined by $g(R) \equiv r \e^{\eta(R/r - 1)}$, where $\eta = -\frac{\log\gamma}{k}$ is a positive constant that depends on the discount parameters $k > 0$ and $\gamma \in (0, 1)$ of the hyperbolic and geometric returns, respectively, and where $r$ is the reference reward.
\end{theorem}

\begin{proof}
Note that the reward transform function $g$ defined above satisfies $g(0) = 0$. Let $t+1$ be the time step of the only non-zero reward $R_{t+1}$. The exponentially-discounted return $G_0$ of the transformed rewards $g(R)$ is then equal to
\begin{align*}
    G_0
    & = \gamma^t g(R_{t+1}) \\
    & = \gamma^t r \e^{\eta(R_{t+1}/r - 1)} \\
    & = (\e^{\log \gamma})^t r \e^{\eta(R_{t+1}/r - 1)}\\
    & = r \e^{t\log \gamma}\e^{\eta(R_{t+1}/r - 1)}\\
    & = r \e^{\eta (R_{t+1}/r - 1) + t \log\gamma}\,.
\end{align*}
This implies that $G_0 > r$ if and only if
\[
\e^{\eta (R_{t+1}/r - 1) + t \log\gamma} > 1 \,,
\]
which is true if and only if $\eta (R_{t+1}/r - 1) + t \log\gamma > 0$. By definition of $\eta = - \frac{\log \gamma}{k}$, this is equivalent to
\begin{align*}
& t \log\gamma - \frac{\log\gamma}{k}\left(\frac{R_{t+1}}{r} - 1\right) > 0 \,,\\
\implies
& t  - \frac{1}{k}\left(\frac{R_{t+1}}{r} - 1\right) < 0 \,,\\
\implies
& t k - \left(\frac{R_{t+1}}{r} - 1\right) < 0 \,,\\
\implies
& \left(\frac{R_{t+1}}{r} - 1\right) > tk \,,\\
\implies
& \frac{R_{t+1}}{r} > 1 + t k \,.\qedhere
\end{align*}
\end{proof}
Interestingly, the predictions made by the transform above will differ from both HDTD and hyperbolic discounting for \emph{dense} rewards and/or \emph{stochastic} rewards. It is also not the only transform, in the more general class of non-linear Bellman equations, that will match the hyperbolic discounts for sparse rewards.  It is an interesting open question whether the predictions for any of these alternatives could perhaps better fit reality better than existing models.

Some readers may find it unintuitive or undesirable that the return $G_0$ as defined in Theorem 1 depends exponentially on the reward.  We note that we can easily add one additional transformation to map the returns into a different space, e.g., $G_0 = \log \sum_{n=0}^\infty \gamma^n g(R_{t+n+1})$.  As long as this outer transform (in this case the $\log$) is monotonic, this does not change the preference orderings, so the conclusion of the theorem (and the equivalence in terms of preference ordering to hyperbolic discounting) still applies.

\subsection{Non-linear discounting}
Of separate interest are non-linear transformations to the bootstrap value, as in $f(r, v) = r + g(v)$.  Linear discounting $g(v) = \gamma v$ is a special case, but we could instead use a non-linear function, which would imply that the discount factor could depend on the value (for instance similar to HDTD).
\begin{figure}
    \centering
    \includegraphics[width=8cm]{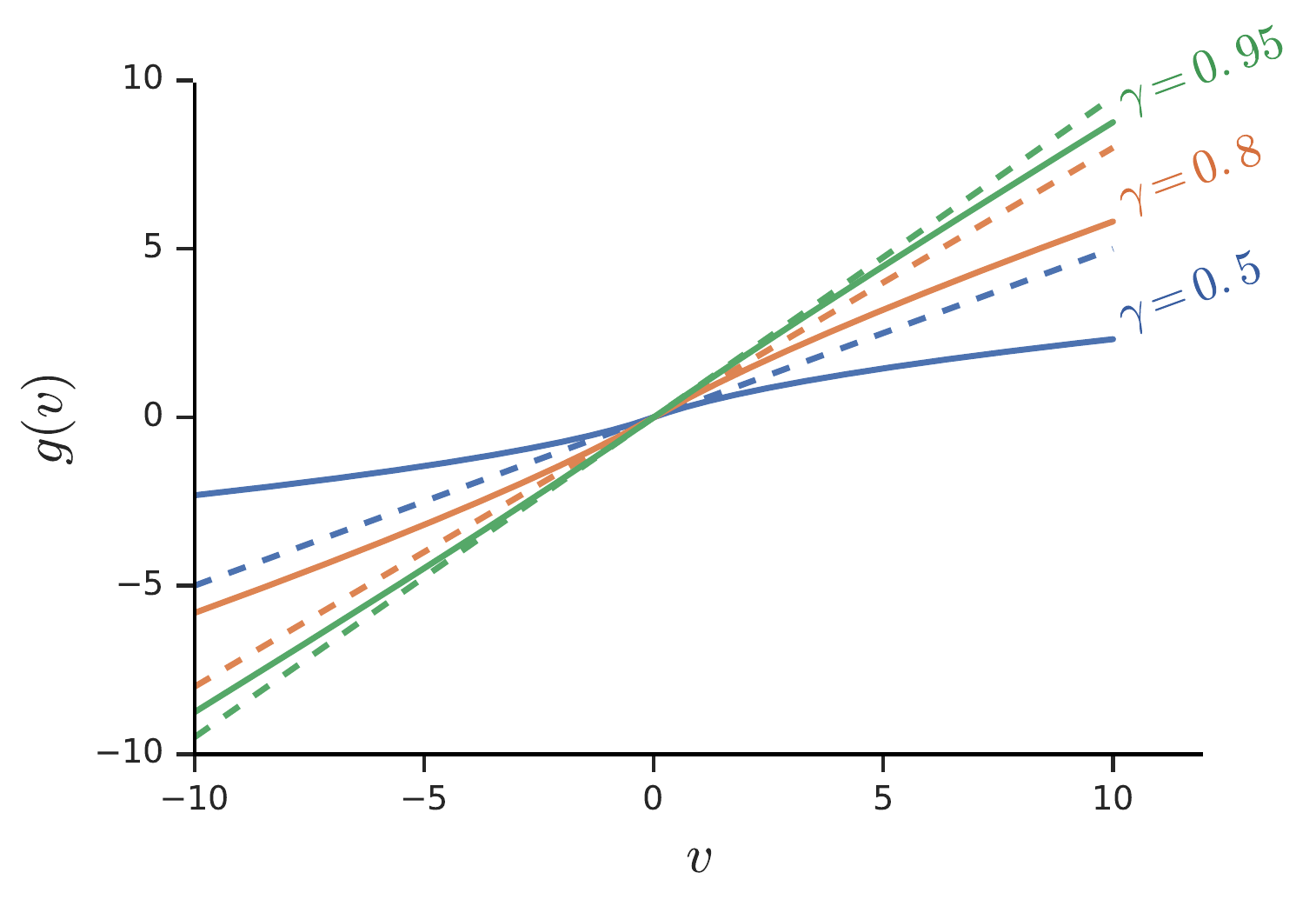}
    \caption{The effect of different non-linear discounting functions.  Dashed lines are linear discounts $g_{\gamma}(v) = \gamma v$, solid lines are non-linear discounts $g_\gamma(v) = \text{sign}(v) ((|v| + 1)^\gamma - 1)$. These functions have a maximal derivative of one at the origin (which we could choose to make smaller by introducing a $\kappa < 1$), and smaller derivatives everywhere else.  This means that these discount lead to contractions.}
    \label{picture}
\end{figure}

We propose to consider functions $g(v)$ that have the following property:
    \[
    \frac{\text d}{\text d v} g(v) \in [0, \kappa] \text{ , where } \kappa \le 1\,.
        \]
This is sufficient for the resulting Bellman operator to be a contraction with a factor $\kappa$, as we prove below.
\begin{theorem}
Let $T^g$ be defined by $(T^g v)(s) = \E{ R_{t+1} + g(v(S_{t+1}) }$. Define $v_* = T^g v_*$.  Then, $\|T^g v - T^gv_*\|_{\infty} \le \kappa \| v - v_* \|$, which means that $v_*$ is the (well-defined) fixed point, and that the operator contracts towards this fixed point with at least a factor $\kappa$.
\end{theorem}
\begin{proof}
\begin{align*}
    &\| T^g v - T^g v_* \|_\infty \\
    & = \| \E{ R_{t+1} + g(v(S_{t+1})) } - \E{ R_{t+1} + g(v_*(S_{t+1})) } \|_\infty \\
    & = \| \E{ g(v(S_{t+1})) - g(v_*(S_{t+1})) } \|_\infty \\
    & \le \| g(v) - g(v_*)\|_\infty \\
    & \le \kappa \| v - v_* \|_\infty \,.
\end{align*}
For the first inequality, we used the fact that the maximum difference between expectations is always at most as large as the maximum difference over all elements in the support of the distribution of the expectation.  For the last inequality, we used the fact that the derivative is bounded by $\kappa$, and therefore the function is $\kappa$-Lipschitz.
\end{proof}
More generally, with a slightly extension of this result, we can say for transforms of the form $f(r, v) = h(r + g(v))$ that these will contract with at least $\kappa = \kappa_h \kappa_g$, where $\kappa_h$ and $\kappa_g$ are the Lipschitz constants of $h$ and $g$, respectively.

To restrict the search space of potentially interesting functions we can, in addition to the property above, consider certain additional restrictions.
We could, for instance, require that the space of functions we consider to be parametrised with a single number $\gamma \in [0, 1]$ (we then denote these functions with $g_{\gamma}$), where we attain the undiscounted linear objective for $\gamma=1$ and a fully myopic objective that only looks at the immediate reward for $\gamma = 0$ at the extremes of the allowed range of $\gamma$. In addition, we may want the function to be symmetrical around the origin, and monotonic.
This means, we might want to consider the following properties.

\begin{enumerate}
    \item $g_0(v)=0, \forall v$ (myopic for $\gamma=0$),
    \item $g_1(v)=v, \forall v$ (undiscounted for $\gamma=1$),
    \item $g_\gamma(v) = -g_\gamma(-v)$ for all $v \in \mathbb{R}$  (symmetric),
    \item $g_\gamma(v) > g_{\gamma'}(v)$ iff $\gamma > \gamma'$, $\forall v > 0$ (monotonic).
\end{enumerate}
The symmetric requirement simply requires $g_{\gamma}$ to be odd, which implies $g_{\gamma}(0) = 0$, for all $\gamma$.

Linear transformations, where $g_\gamma(v) = \gamma v$, share these properties, but they allow for more general non-linear transformations.
As an example, consider the following class that we will call \emph{power-discounting}: $\kappa ((v + 1)^\gamma - 1)$ for $v \ge 0$ and $-\kappa((-v + 1)^\gamma - 1)$ for $v < 0$.  For large $v$, this function becomes very similar to $\kappa v^\gamma$ (or $-\kappa(-v)^\gamma$ for negative $v$), but it has the desired properties enumerated above.  In particular, its derivative is $\kappa/(|x| + 1)^{1 - \gamma}$, which tends to $\kappa/(|x|)^{1 - \gamma}$ for large $|x|$, which implies that larger values will contract faster, and is at most $\kappa$, for $x = 0$. Some examples are shown in Figure \ref{picture}.

Note that the value of a state under non-linear discounting depends on the stochasticity in the environment's dynamics and in the agent's policy, as this stochasticity interacts with non-linear transformations such as power discounting in ways that might not be immediately obvious. To illustrate this, consider the action values under power discounting in a state $s$, from which two actions $a$ and $b$ are available for the agent to select. In both cases, the immediate reward is zero, but selecting action $a$ leads deterministically to a state $x$ with value $v(x)=1$, while selecting action $b$ either leads to a state $y$ with value $v(y) = 2/p$, or it terminates. The action values under linear discounting at $q(s, a) = \gamma$ and $q(s, b) = 2 \gamma$. For any $\gamma > 0$, $q(s, b) > q(s, a)$---action $b$ is always optimal. Under power discounting the agent's preference between the two actions may reverse for large $p$: the agent can become risk averse. In Figure \ref{risk-aversion} we plot the action gaps (y-axis) for multiple values of $\gamma$ (x-axis) and $p$ (different lines). Note that for small $p$ the action gaps can be negative: the agent prefers the certain value of state $v(x) = 1$ over the uncertain value of $v(y) = 2/p$, even if the latter has a larger (undiscounted) expected value.

\begin{figure}
    \centering
    \includegraphics[width=8cm]{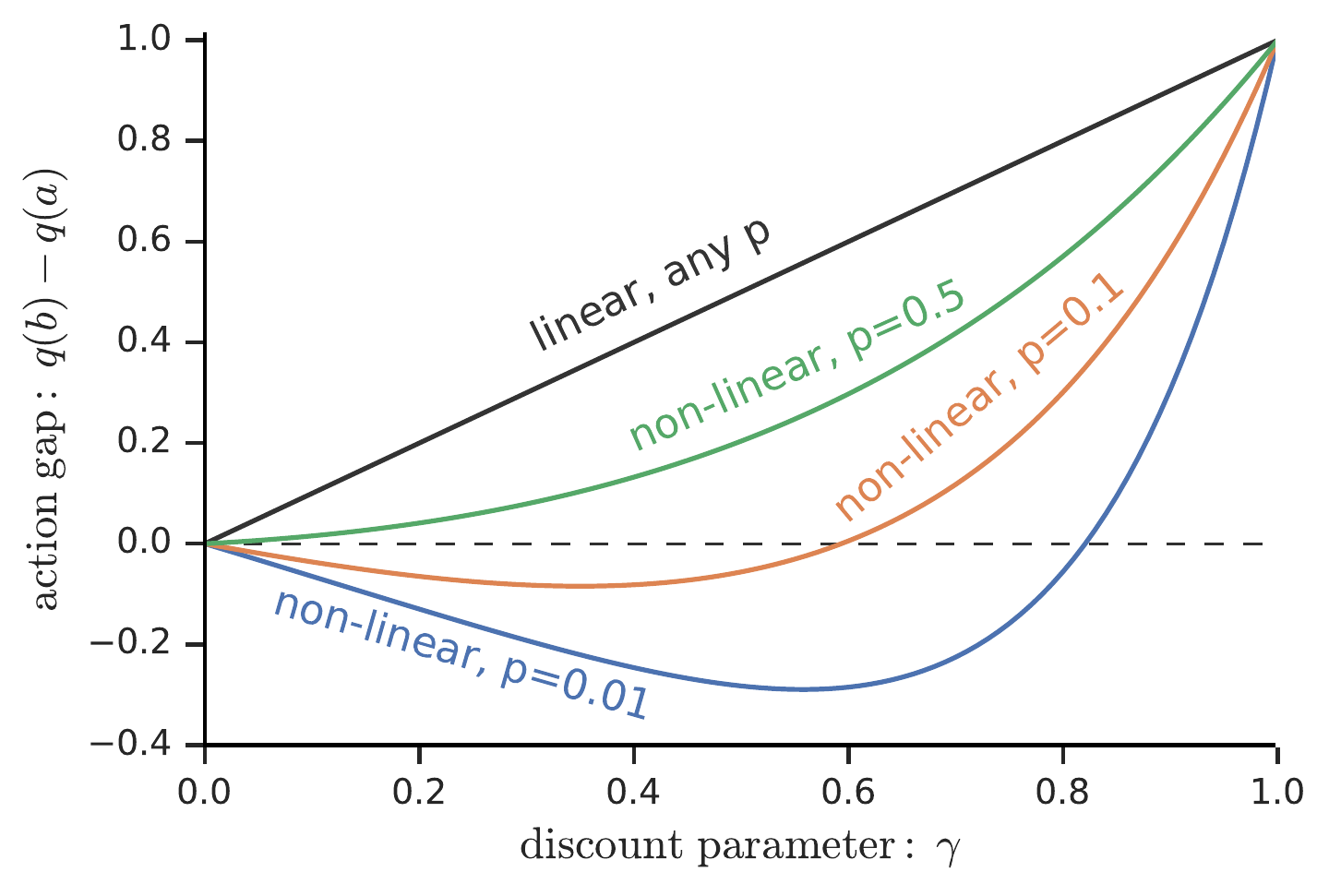}
    \caption{Action gaps for different discounts.  The y-axis is the difference between the values of actions $b$ and $a$. Action $b$ that either terminates immediately (with zero reward) or leads, with probability $p$, to a state with to a state $y$ with value $2/p$. Action $a$ always leads to a state $x$ with value $v(x)=1$. The discount parameter $\gamma$ is on the x-axis. The different lines correspond to different problems, with different transition probabilities $p$.  Note that for power-discounting the lines can go below zero, but not for linear discounting.}
    \label{risk-aversion}
\end{figure}

\section{Prediction and control performance}
In addition to a larger design space to model natural phenomena, like hyperbolic discounting, non-linear Bellman equations may offer an interesting path towards algorithms that work well for prediction or control.  This is similar to how it is common practice to add a discount factor, even if we actually mostly care about the undiscounted returns, simply because the resulting performance is better.
Apart from a few hints in the literature \citep[e.g.][]{Pohlen:2018, Kapturowski:2019}, this is still a relatively unexplored area, and this seems an interesting avenue for future research.

Non-linear Bellman equations allow us to capture rich and varied information about the world. It is thought to be important for our agents to learn many things \citep{Sutton:2011}, and non-linear Bellman equations offer a rich and powerful toolbox to express even more varied predictive questions.

\small

\bibliography{references}

\begin{thebibliography}{29}
\providecommand{\natexlab}[1]{#1}
\providecommand{\url}[1]{\texttt{#1}}
\expandafter\ifx\csname urlstyle\endcsname\relax
  \providecommand{\doi}[1]{doi: #1}\else
  \providecommand{\doi}{doi: \begingroup \urlstyle{rm}\Url}\fi

\bibitem[Ainslie and Herrnstein(1981)]{Ainslie:1981}
G.~Ainslie and R.~J. Herrnstein.
\newblock Preference reversal and delayed reinforcement.
\newblock \emph{Animal Learning {\&} Behavior}, 9\penalty0 (4):\penalty0
  476--482, 1981.

\bibitem[Alexander and Brown(2010)]{Alexander:2010}
W.~H. Alexander and J.~W. Brown.
\newblock Hyperbolically discounted temporal difference learning.
\newblock \emph{Neural computation}, 22\penalty0 (6):\penalty0 1511--1527,
  2010.

\bibitem[Bellemare et~al.(2013)Bellemare, Naddaf, Veness, and
  Bowling]{Bellemare:2013}
M.~G. Bellemare, Y.~Naddaf, J.~Veness, and M.~Bowling.
\newblock The arcade learning environment: An evaluation platform for general
  agents.
\newblock \emph{J. Artif. Intell. Res. {(JAIR)}}, 47:\penalty0 253--279, 2013.

\bibitem[Bellemare et~al.(2017)Bellemare, Dabney, and Munos]{Bellemare:2017}
M.~G. Bellemare, W.~Dabney, and R.~Munos.
\newblock A distributional perspective on reinforcement learning.
\newblock In \emph{Proceedings of the 34th International Conference on Machine
  Learning-Volume 70}, pages 449--458, 2017.

\bibitem[Bellman(1957)]{Bellman:57}
R.~Bellman.
\newblock \emph{Dynamic Programming}.
\newblock Princeton University Press, 1957.

\bibitem[Bertsekas and Tsitsiklis(1996)]{Bertsekas:96}
D.~P. Bertsekas and J.~N. Tsitsiklis.
\newblock \emph{Neuro-dynamic Programming}.
\newblock Athena Scientific, Belmont, MA, 1996.

\bibitem[Dabney et~al.(2018)Dabney, Rowland, Bellemare, and Munos]{Dabney:2018}
W.~Dabney, M.~Rowland, M.~G. Bellemare, and R.~Munos.
\newblock Distributional reinforcement learning with quantile regression.
\newblock In \emph{Thirty-Second AAAI Conference on Artificial Intelligence},
  2018.

\bibitem[Fedus et~al.(2019)Fedus, Gelada, Bengio, Bellemare, and
  Larochelle]{Fedus:2019}
W.~Fedus, C.~Gelada, Y.~Bengio, M.~G. Bellemare, and H.~Larochelle.
\newblock Hyperbolic discounting and learning over multiple horizons.
\newblock \emph{arXiv preprint arXiv:1902.06865}, 2019.

\bibitem[Green et~al.(1994)Green, Fristoe, and Myerson]{Green:1994}
L.~Green, N.~Fristoe, and J.~Myerson.
\newblock Temporal discounting and preference reversals in choice between
  delayed outcomes.
\newblock \emph{Psychonomic Bulletin \& Review}, 1\penalty0 (3):\penalty0
  383--389, 1994.

\bibitem[Hessel et~al.(2018)Hessel, Modayil, Van~Hasselt, Schaul, Ostrovski,
  Dabney, Horgan, Piot, Azar, and Silver]{Hessel:2018}
M.~Hessel, J.~Modayil, H.~Van~Hasselt, T.~Schaul, G.~Ostrovski, W.~Dabney,
  D.~Horgan, B.~Piot, M.~Azar, and D.~Silver.
\newblock Rainbow: Combining improvements in deep reinforcement learning.
\newblock In \emph{Thirty-Second AAAI Conference on Artificial Intelligence},
  2018.

\bibitem[Howard(1960)]{Howard:1960}
R.~A. Howard.
\newblock \emph{Dynamic programming and {Markov} processes}.
\newblock MIT Press, 1960.

\bibitem[Kapturowski et~al.(2019)Kapturowski, Ostrovski, Quan, Munos, and
  Dabney]{Kapturowski:2019}
S.~Kapturowski, G.~Ostrovski, J.~Quan, R.~Munos, and W.~Dabney.
\newblock Recurrent experience replay in distributed reinforcement learning.
\newblock In \emph{International Conference on Learning Representations}, 2019.

\bibitem[Kurth-Nelson and Redish(2009)]{Kurth:2009}
Z.~Kurth-Nelson and A.~D. Redish.
\newblock Temporal-difference reinforcement learning with distributed
  representations.
\newblock \emph{PLoS One}, 4\penalty0 (10), 2009.

\bibitem[Mnih et~al.(2015)Mnih, Kavukcuoglu, Silver, Rusu, Veness, Bellemare,
  Graves, Riedmiller, Fidjeland, Ostrovski, Petersen, Beattie, Sadik,
  Antonoglou, King, Kumaran, Wierstra, Legg, and Hassabis]{Mnih:2015}
V.~Mnih, K.~Kavukcuoglu, D.~Silver, A.~A. Rusu, J.~Veness, M.~G. Bellemare,
  A.~Graves, M.~Riedmiller, A.~K. Fidjeland, G.~Ostrovski, S.~Petersen,
  C.~Beattie, A.~Sadik, I.~Antonoglou, H.~King, D.~Kumaran, D.~Wierstra,
  S.~Legg, and D.~Hassabis.
\newblock Human-level control through deep reinforcement learning.
\newblock \emph{Nature}, 518\penalty0 (7540):\penalty0 529--533, 2015.

\bibitem[Pohlen et~al.(2018)Pohlen, Piot, Hester, Azar, Horgan, Budden,
  Barth{-}Maron, van Hasselt, Quan, Vecer{\'{\i}}k, Hessel, Munos, and
  Pietquin]{Pohlen:2018}
T.~Pohlen, B.~Piot, T.~Hester, M.~G. Azar, D.~Horgan, D.~Budden,
  G.~Barth{-}Maron, H.~van Hasselt, J.~Quan, M.~Vecer{\'{\i}}k, M.~Hessel,
  R.~Munos, and O.~Pietquin.
\newblock Observe and look further: Achieving consistent performance on atari.
\newblock \emph{CoRR}, abs/1805.11593, 2018.

\bibitem[Puterman(1994)]{Puterman:1994}
M.~L. Puterman.
\newblock \emph{Markov Decision Processes: Discrete Stochastic Dynamic
  Programming}.
\newblock John Wiley \& Sons, Inc. New York, NY, USA, 1994.

\bibitem[Rowland et~al.(2019)Rowland, Dadashi, Kumar, Munos, Bellemare, and
  Dabney]{Rowland:2019}
M.~Rowland, R.~Dadashi, S.~Kumar, R.~Munos, M.~G. Bellemare, and W.~Dabney.
\newblock Statistics and samples in distributional reinforcement learning.
\newblock In \emph{International Conference on Machine Learning}, pages
  5528--5536, 2019.

\bibitem[Sutton(1988)]{Sutton:1988}
R.~S. Sutton.
\newblock Learning to predict by the methods of temporal differences.
\newblock \emph{Machine learning}, 3\penalty0 (1):\penalty0 9--44, 1988.

\bibitem[Sutton(1995)]{Sutton:1995}
R.~S. Sutton.
\newblock {TD} models: Modeling the world at a mixture of time scales.
\newblock In \emph{Machine Learning Proceedings 1995}, pages 531--539.
  Elsevier, 1995.

\bibitem[Sutton and Barto(2018)]{SuttonBarto:2018}
R.~S. Sutton and A.~G. Barto.
\newblock \emph{Reinforcement Learning: An Introduction}.
\newblock The MIT press, Cambridge MA, 2018.

\bibitem[Sutton et~al.(2009)Sutton, Maei, Precup, Bhatnagar, Silver,
  Szepesv{\'a}ri, and Wiewiora]{Sutton:2009icml}
R.~S. Sutton, H.~R. Maei, D.~Precup, S.~Bhatnagar, D.~Silver,
  C.~Szepesv{\'a}ri, and E.~Wiewiora.
\newblock Fast gradient-descent methods for temporal-difference learning with
  linear function approximation.
\newblock In \emph{Proceedings of the 26th Annual International Conference on
  Machine Learning (ICML 2009)}, pages 993--1000. ACM, 2009.

\bibitem[Sutton et~al.(2011)Sutton, Modayil, Delp, Degris, Pilarski, White, and
  Precup]{Sutton:2011}
R.~S. Sutton, J.~Modayil, M.~Delp, T.~Degris, P.~M. Pilarski, A.~White, and
  D.~Precup.
\newblock Horde: A scalable real-time architecture for learning knowledge from
  unsupervised sensorimotor interaction.
\newblock In \emph{The 10th International Conference on Autonomous Agents and
  Multiagent Systems-Volume 2}, pages 761--768. International Foundation for
  Autonomous Agents and Multiagent Systems, 2011.

\bibitem[Thaler(1981)]{Thaler:1981}
R.~Thaler.
\newblock Some empirical evidence on dynamic inconsistency.
\newblock \emph{Economics letters}, 8\penalty0 (3):\penalty0 201--207, 1981.

\bibitem[van Hasselt and Sutton(2015)]{vanHasseltSutton:2015}
H.~van Hasselt and R.~S. Sutton.
\newblock Learning to predict independent of span.
\newblock \emph{CoRR}, abs/1508.04582, 2015.

\bibitem[van Hasselt et~al.(2016{\natexlab{a}})van Hasselt, Guez, Hessel, Mnih,
  and Silver]{vanHasselt:2016PopArt}
H.~van Hasselt, A.~Guez, M.~Hessel, V.~Mnih, and D.~Silver.
\newblock Learning values across many orders of magnitude.
\newblock In \emph{Advances in Neural Information Processing Systems}, pages
  4287--4295, 2016{\natexlab{a}}.

\bibitem[van Hasselt et~al.(2016{\natexlab{b}})van Hasselt, Guez, and
  Silver]{vanHasselt:2016DDQN}
H.~van Hasselt, A.~Guez, and D.~Silver.
\newblock Deep reinforcement learning with {Double Q-learning}.
\newblock \emph{AAAI}, 2016{\natexlab{b}}.

\bibitem[Wang et~al.(2016)Wang, de~Freitas, Schaul, Hessel, van Hasselt, and
  Lanctot]{Wang:2016}
Z.~Wang, N.~de~Freitas, T.~Schaul, M.~Hessel, H.~van Hasselt, and M.~Lanctot.
\newblock Dueling network architectures for deep reinforcement learning.
\newblock In \emph{International Conference on Machine Learning}, New York, NY,
  USA, 2016.

\bibitem[Watkins(1989)]{Watkins:1989}
C.~J. C.~H. Watkins.
\newblock \emph{Learning from delayed rewards}.
\newblock PhD thesis, University of Cambridge England, 1989.

\bibitem[Xu et~al.(2018)Xu, van Hasselt, and Silver]{Xu:2018}
Z.~Xu, H.~P. van Hasselt, and D.~Silver.
\newblock Meta-gradient reinforcement learning.
\newblock In \emph{Advances in Neural Information Processing Systems}, pages
  2402--2413, 2018.

\end{thebibliography}
\bibliographystyle{abbrvnat}

\end{document}